\documentclass[10pt,twocolumn,letterpaper]{article}

\usepackage{cvpr}              

\usepackage{graphicx}
\usepackage{amsmath}
\usepackage{amssymb}
\usepackage{booktabs}

\usepackage[pagebackref,breaklinks,colorlinks]{hyperref}

\usepackage{hyperref}
\usepackage{url}
\usepackage{amsthm}
\usepackage{multirow}
\usepackage{array}
\usepackage{subcaption}
\usepackage{float}
\usepackage{nicefrac}
\usepackage{wrapfig}
\usepackage[export]{adjustbox}
\newtheorem{lemma}{Lemma}

\usepackage[capitalize]{cleveref}
\crefname{section}{Sec.}{Secs.}
\Crefname{section}{Section}{Sections}
\Crefname{table}{Table}{Tables}
\crefname{table}{Tab.}{Tabs.}

\def\cvprPaperID{57} 
\def\confName{CVPR}
\def\confYear{2023}

\begin{document}

\title{Exploring Compositional Visual Generation with Latent Classifier Guidance}

\newcommand*\samethanks[1][\value{footnote}]{\footnotemark[#1]}
\author{
Changhao Shi$^1$\thanks{Work done during the internship at NEC Laboratories America.}
\quad Haomiao Ni$^2$\samethanks
\quad Kai Li$^4$\quad Shaobo Han$^4$\quad Mingfu Liang$^3$\samethanks \quad Martin Renqiang Min$^4$ \\
$^1$University of California, San Diego, CA, USA \\ 
$^2$The Pennsylvania State University, University Park, PA, USA \\
$^3$Northwestern University, Evanston, IL, USA \\
$^4$NEC Laboratories America, Princeton, NJ, USA \\
$^1${\tt\small cshi@ucsd.edu}
\quad$^2${\tt\small hfn5052@psu.edu}
\quad$^3${\tt\small mingfuliang20202@u.northwestern.edu}\\
\quad$^4${\tt\small\{kaili, shaobo, renqiang\}@nec-labs.com}
}
\maketitle

\begin{abstract}
Diffusion probabilistic models have achieved enormous success in the field of image generation and manipulation.
In this paper, we explore a novel paradigm of using the diffusion model and classifier guidance in the latent semantic space for compositional visual tasks. 
Specifically, we train latent diffusion models and auxiliary latent classifiers to facilitate non-linear navigation of latent representation generation for any pre-trained generative model with a semantic latent space. 
We demonstrate that such conditional generation achieved by latent classifier guidance provably maximizes a lower bound of the conditional log probability during training. 
To maintain the original semantics during manipulation, we introduce a new guidance term, which we show is crucial for achieving compositionality. 
With additional assumptions, we show that the non-linear manipulation reduces to a simple latent arithmetic approach.
We show that this paradigm based on latent classifier guidance is agnostic to pre-trained generative models, and present competitive results for both image generation and sequential manipulation of real and synthetic images. 
Our findings suggest that latent classifier guidance is a promising approach that merits further exploration, even in the presence of other strong competing methods.
\end{abstract}

\section{Introduction}
\label{sec:intro}

In recent years, the machine learning and computer vision communities have witnessed great progress in the field of deep generative modeling. From variational autoencoders (VAEs) \cite{kingma2013auto}, normalizing flows \cite{rezende2015variational}, and generative adversarial networks (GANs) \cite{goodfellow2014generative,brock2018large, choi2020stargan,abdal2020image2stylegan++,wu2021stylespace}, to the very recent diffusion probabilistic models~\cite{sohl2015deep,ho2020denoising, song2020denoising, nichol2021improved,abstreiter2021diffusion, rombach2022high} and score-based models \cite{song2019generative,song2020score}, generating high-quality, realistic images has become easier, if not impossible before. Despite the previous significant progress, controlling the generation process using various conditions, such as class labels and text descriptions, still remains challenging.

One major difficulty towards such controllable generation is compositionality.
Compositionality in generative modeling, or compositional generation, is the ability of a conditional generative model to produce realistic outputs given multiple conditions and their relations.
Broadly speaking, there exist two types of methods for achieving such compositionality.
The first class of methods tackles the problem directly in the image space, either relying on energy-based models (EBMs) or drawing inspiration from them \cite{du2020compositional,liu2022compositional}.
Such technique is referred as classifier-free guidance in the literature, in contrast to classifier guidance that relies on auxiliary image classifiers \cite{nichol2021improved}.
Unlike classifier guidance that is mainly used for controllable generation with a single condition, the classifier-free guidance is naturally composable and suitable for multiple conditions.
However, such methods can not leverage the nice properties of the latent space such as disentanglement \cite{bengio2013representation}. 
Also, training multiple image space sub-models, either EBMs or diffusion models, can be cumbersome, especially when the number of conditions grows.

The second class of methods focuses on the latent space of pre-trained generative models.
These methods aim to find a rule that governs the manipulation of latent codes so as to obtain outputs with desired properties.
When the latent space is disentangled as in StyleGAN \cite{karras2019style,karras2020analyzing}, linear control is possible by carefully identifying and combining the latent directions of each attribute \cite{shen2020interpreting,wu2021stylespace,harkonen2020ganspace,shen2021closed}.
While this is not new, the feasibility of such linear control in the context of compositionality is still under-explored.
On the other hand, non-linear manipulations of the latent space have also been proposed for finer control, in the sense that each modification will be customized for each latent code.
However, the previous non-linear methods are either not amendable for new attributes \cite{abdal2021styleflow} or not agnostic to various models and latent spaces 
 \cite{nie2021controllable}.
 Using diffusion models to control latent space through classifier-free guidance has been widely recognized \cite{rombach2022high}, but not within the context of compositionality.
 Using diffusion models and classifier guidance in latent space however, is still a missing piece.

In this paper, we aim to fill in this missing piece and answer the question: is the latent diffusion model with latent classifier guidance useful for compositional image generation and manipulation\footnote{For the sake of clarity, the term ``compositional generation" will refer to generating images conditioning only on attributes while the term ``compositional manipulation" will specifically refer to conditioning on both attributes and an original image.}?
We demonstrate that classifier guidance can help diffusion probabilistic models to manipulate latent spaces in a non-linear way, 
and this process can be further simplified to a linear version that resembles vector arithmetic-based manipulation with additional assumptions.
For compositional generation, we train latent diffusion models and auxiliary latent classifiers for pre-trained generators, and use classifier guidance to sample in the latent space. 
To facilitate the manipulation of synthetic or real images, we introduce an additional guidance term by framing the problem as incorporating a source image condition into the compositional generation process.
We demonstrate that employing latent classifier guidance with diffusion models maximizes a lower bound of the conditional log probability function, providing a provable approach for conditioning on multiple attributes. 
Our experiments validate the effectiveness of this technique, as it can generate realistic images with various attribute compositions and manipulate both \textbf{synthetic and real} images in a coherent manner.
We also find that the linear version of our proposed method based on vector arithmetic can serve as a strong baseline in many scenarios, 
despite previous studies focusing on non-linear manipulation.


\section{Related Work}
\label{sec:related_works}
\paragraph{(Conditional) diffusion models.}
Diffusion models have become increasingly favorable over other generative models such as GANs \cite{goodfellow2014generative} and VAEs \cite{kingma2013auto} due to their photo-realistic generation quality and ease of training.
\cite{sohl2015deep} proposed the first functional framework of diffusion models from the perspective of thermodynamics, then this framework was followed by \cite{song2019generative,song2020score,ho2020denoising,song2020denoising} whose works established the foundation of diffusion models that we see today.
For the conditional generation with diffusion models, \cite{dhariwal2021diffusion} further formulated classifier guidance and lifted the generation quality of diffusion models over previous state-of-the-art GANs.
\cite{ho2022classifier} then proposed classifier-free guidance which nowadays is used in many large-scale image generation engine \cite{nichol2021glide,ramesh2022hierarchical,saharia2022photorealistic}.
Although most of these diffusion models work on the image space, recently latent diffusion models have also drawn great attention and achieved remarkable results \cite{vahdat2021score,abstreiter2021diffusion,rombach2022high}.
\vspace{-10pt}
\paragraph{Compositionality in latent space.}
Due to the wide success of StyleGANs \cite{karras2019style,karras2020analyzing} in image generation, most efforts on conditional generation have been focusing on the latent space of StyleGANs.
To use linear arithmetic for manipulation, various ways of finding attribute directions have been proposed.
\cite{shen2020interpreting} found a direction for each attribute by training linear SVMs, then perturbed the latent point along the orthogonal projection of these directions to prevent unwanted semantic changes.
\cite{wu2021stylespace} detected latent channels that only allow local changes for specific attributes.
Such directions can also be identified in an unsupervised fashion, using the PCA decomposition of the latent space \cite{harkonen2020ganspace}, or the SVD of the first subsequent linear layer \cite{shen2021closed}.
On the non-linear side, \cite{abdal2021styleflow} designed a framework to control a set of pre-decided attributes using conditional normalizing flow.
\cite{nie2021controllable} used latent EBMs to control the style generation with non-linear classifiers.
Note that none of the linear methods explicitly discussed the composition of multiple attributes and their relations as in the non-linear methods.
\vspace{-10pt}
\paragraph{Compositionality in image space.}
Some other methods tackled compositional generation in the image space.
These methods either directly used EBMs or employed diffusion models that can be considered as EBMs.
\cite{du2020compositional} trained EBMs for each condition and composed them together by defining new energy functions based on each individual energy function and their relations.
\cite{liu2022compositional} followed their proposal and adopted classifier-free guidance in diffusion models to multiple conditions.
However, these image space models can not leverage the disentanglement property of the latent space, and training EBMs or diffusion models for each condition can be cumbersome.
Note that although some large-scale text-to-image generation engines \cite{nichol2021glide,ramesh2022hierarchical,saharia2022photorealistic} claim compositionality in their methods, they do not model compositionality explicitly but rather rely on implicit composition by their language models, which leads to less satisfying results when the set of conditions gets large \cite{liu2022compositional}.


\section{Methodology}
\label{sec:method}

In this section, we will describe how the latent diffusion model and classifier guidance can be used to generate and manipulate images in a principled way.

\subsection{Latent Diffusion Modeling}

A diffusion model is a type of deep latent variable model that approximates an unknown data distribution $p(x)$ through smooth, iterative denoising steps. It maps a pre-defined noise distribution to the data distribution using the following formula: $p_{\theta}(x_{0}) = \int p_{\theta}(x_{0:T}) d x_{1:T}$, where $x_{1:T}$ are the latent variables with the same dimensionality as the data $x_0$.
The forward diffusion process, resembling a parameter-free encoder, is a Markov chain $q(x_{1:T}|x_0) = \prod_{t=1}^T q(x_t|x_{t-1})$, where each $q(x_t|x_{t-1})$ is typically a Gaussian distribution.
The forward process perturbs inputs according to a pre-defined schedule, and the transformed data distribution $q(x_t|x_0)$ will gradually converge to a standard Gaussian $\mathcal{N}(x_T; \mathbf{0}, \mathbf{I})$.
The reverse sampling process, resembling a hierarchical decoder, is composed of a sequence of de-noising steps $p_{\theta}(x_{t-1}|x_t)$, which is parameterized by a deep neural network with parameter $\theta$.

During training, the input images are corrupted by the forward process, and the diffusion model is trained to 
reconstruct the original images from the corrupted inputs.
Specifically, for de-noising diffusion probabilistic models (DDPM) \cite{ho2020denoising}, the training objective is formulated as a re-weighted variational bound by treating DDPMs as VAEs, while for scored-based generative models \cite{song2019generative}, the objective is derived using score matching.
Once trained, to generate samples from the learned distribution, one first samples $x_T$ from a standard Gaussian and then uses the reverse process to transform it into the image space.

Here, we focus on leveraging the latent space of a pre-trained generative model. Specifically, we train a diffusion model to approximate the latent distribution $p(z)$ of a pre-trained generator $G$ that maps a latent space $\mathcal{Z}$ to the image space $\mathcal{X}$. Modeling the latent space has several advantages over modeling the image space. For instance, the latent space enjoys properties such as disentanglement \cite{bengio2013representation}, which can facilitate more controllable manipulations of the generated images. Additionally, using various guidance techniques in the latent space is often more feasible since training latent guidance terms is generally easier than training other manipulation methods in image space \cite{shen2020interpreting}. 

\subsection{Conditional and Compositional Generation}

Conditional generation with diffusion models relies on perturbing unconditional generation with user-specified guidance terms, namely classifier guidance \cite{sohl2015deep,song2020score,dhariwal2021diffusion} and classifier-free guidance \cite{ho2022classifier}.
Although classifier-free guidance performs competitively in image space and is sometimes more favorable than classifier guidance \cite{ho2022classifier,liu2022compositional}, we argue that using classifier guidance in latent diffusion models has its unique advantages.
Regarding the under-performance of classifier guidance in the image space, one popular suspicion is that image classifiers tend to learn shortcuts from suspicious correlations. 
For example, a deep neural network classifier on the attribute ``old'' can be misguided by ``white hair'' and ignore its holistic features.
This problem is alleviated in a compact, even disentangled latent space, if the semantic directions of `old' and `white hair' are orthogonal.
Also, deep image classifiers are typically vulnerable to adversarial attacks, while latent classifiers with much few parameters suffer less from this problem.
Another benefit is that classifiers are usually easier to train than diffusion models used in classifier-free guidance.
Finally, when the classifiers are linear, classifier guidance resembles linear arithmetic methods, as we will show in Section~\ref{Connection to Linear Arithmetic}.

The goal of conditional generation is to model the conditional distribution $p(z|y)$ where $y$ is the conditions or attributes.
By Bayes rules $p(z_t|y) = \nicefrac{p(z_t)p(y|z_t)}{p(y)}$, the score of the conditional probability $\nabla_{z_t} \log p(z_{t}|y)$ can be factorized as the unconditional score $\nabla_{z_t} \log p(z_t)$ and the gradient flow $\nabla_{z_t} \log p(y|z_t)$.
Therefore, one simply needs an unconditional latent diffusion model and a latent classifier to model the conditional score, known as classifier guidance.
In practice, the classifier guidance term is usually scaled by a factor $\alpha$, such that $\nabla_{z_t} \log p(z_t|y) = \nabla_{z_t} \log p(z_t) + \alpha \nabla_{z_t} \log p(y|z_t)$.
The factor $\alpha$ serves as a temperature parameter which adds another layer of controllability to the sharpness of the posterior distribution $p(y|z_t)$.

Compositional generation can be considered as conditional generation with multiple conditions and the relations among them.
In this paper, we consider two relations, conjunction ``AND'' and negation ``NOT''. 
For the conjunction of attributes $y^1 \wedge y^2 \wedge ... \wedge y^n$, assuming the conditions to be independent of each other, we can simply factorize the \textit{compositional} log probability as 
\begin{multline}
\label{eq:and}
    \nabla_{z_t} \log p(z_t|y^1,y^2,...,y^n) = \\
    \nabla_{z_t} \log p(z_t) + \sum_{i=1}^n \alpha_t^i \nabla_{z_t} \log p(y^i|z_t).
\end{multline}
And with attribute negations $y^1 \wedge ... \wedge y^{m-1} \wedge \overline{y^m} \wedge ... \wedge \overline{y^n}$, without loss of generality, we can factorize the log probability similarly
\begin{multline}
\label{eq:not}
    \nabla_{z_t} \log p(z_t|y^1,...,y^n) = \nabla_{z_t} \log p(z_t) +\\
    \sum_{i=1}^{m-1} \alpha_t^i \nabla_{z_t} \log p(y^i|z_t) - \sum_{i=m}^n \beta_t^i \nabla_{z_t} \log p(y^i|z_t).
\end{multline}

While classifier guidance is useful for compositional generation, there is no guarantee that those results will be anything similar to the original image when doing manipulations.
This is because the generation is not conditioned on the original image.
As there is no constraint on the specific form of the posterior \cite{sohl2015deep}, conditioning on the original image amounts to adding a new guidance term $\gamma_t \nabla_z \log p(\hat{z}|z)$, where $\hat{z}$ is the latent of the image to be manipulated.
For conjunction relations as in Eq.~(\ref{eq:and}), the overall score function for manipulation then becomes 
\begin{multline}
\label{eq:reg}
    \nabla_{z_t} \log p({z_{t}}|y^1,y^2,...,y^n,\hat{z}) = \nabla_{z_t} \log p(z_t) + \\
    \sum_{i=1}^n \alpha_t^i \nabla_{z_t} \log p(y^i|z_t) + \gamma_t \nabla_{z_t} \log p(\hat{z}|z_t),
\end{multline}
similarly for Eq.~(\ref{eq:not}) with the presence of negation.
When $p(\hat{z}|z_t)$ is modeled by an isotropic Gaussian distribution, the new guidance term $\gamma_t \nabla_z \log p(\hat{z}|z)$ behaves as a regularization term $\nabla_{z_t} {\|z_t-\hat{z}\|}_2^2$.

\subsection{Model Training}

A true compositional model should be able to easily encompass new attributes without re-training the whole model.
Indeed, the training of the unconditional diffusion models and the latent classifiers can be decoupled, and such training amounts to maximizing the evidence lower bound (ELBO) of the conditional log-likelihood.
This means that encompassing new attributes simply requires training classifiers on them, and the latent diffusion model as well as used classifiers can be recycled.

We take DDPM as our example and begin with unconditional generation.
\begin{lemma}
\label{lemma:1}
    The unconditional ELBO of DDPM is given by the following equation:
    \begin{multline}
        \label{eq:unconditional_elbo}
        \mathcal{L}_{uncond} :=
        \mathbb{E}_{q(z_{1:T}|z_0)}\bigg[\log \frac{p(z_T)}{q(z_T|z_0)} + \\
        \sum_{t=2}^T \log \frac{p(z_{t-1}|z_t)}{q(z_{t-1}|z_t,z_0)} + \log p(z_0|z_1)\bigg].
    \end{multline}
\end{lemma}
See \cite{ho2020denoising} for the detailed proof.

\begin{lemma}[Compositional generation and manipulation]
\label{lemma:2}
    The conditional ELBO of DDPM with condition $y$ is given by:
    \begin{multline}
    \label{eq:conditional_elbo}
        \mathbb{E}_{q(z_{1:T}|z_0)}\bigg[\sum_{t=1}^T\log p(y|z_{t-1})\bigg] + \mathcal{L}_{uncond} + C,
    \end{multline}
    and with independent conditions $\{y^1,y^2,...,y^n\}$ and $\hat{z}$, the ELBO is given by:
\begin{multline}
    \mathbb{E}_{q(z_{1:T}|x_0)}
    \Bigg[\sum_{t=1}^T\bigg[\sum_{i=1}^n \log p(y^i|z_{t-1}) + \log p(\hat{z}|z_{t-1})\bigg]\Bigg]\\
    + \mathcal{L}_{uncond} + C.
\end{multline}
\end{lemma}

\begin{proof}
Lemma.~\ref{lemma:2} can be proved using $p({z}_{t-1}|{z}_t,y) = Z p({z}_{t-1}|{z}_t)p(y|{z}_{t-1})$ ($Z$ is a normalizing constant) and following the same routine as the proof of Lemma.~\ref{lemma:1}.
\allowdisplaybreaks
    \begin{align*}
        & \log p({z}_0,y) \\
        = & \log \int p(z_{0:T}|y)p(y) d z_{1:T} \\
        \geq & \mathbb{E}_{q(z_{1:T}|z_0)} \log \frac{p(z_{0:T}|y)p(y)}{q(z_{1:T}|z_0)}\\
        = & \mathbb{E}_{q(z_{1:T}|z_0)}\bigg[\log \frac{p(z_T)}{q(z_T|z_0)} + \sum_{t=2}^T \log \frac{p(z_{t-1}|z_t,y)}{q(z_{t-1}|z_t,z_0)}\\
        & + \log p(z_0|z_1,y)\bigg] + C_1\\
        = & \mathbb{E}_{q(z_{1:T}|z_0)}\bigg[\log \frac{p(z_T)}{q(z_T|z_0)} + \sum_{t=2}^T \log \frac{p(z_{t-1}|z_t)}{q(z_{t-1}|z_t,z_0)} \\
        & + \log p(z_0|z_1) + \sum_{t=1}^T \log p(y|z_{t-1})\bigg] + C_2 \\
        = & \mathcal{L}_{uncond} + \mathbb{E}_{q(z_{1:T}|z_0)} \bigg[\sum_{t=1}^T \log p(y|z_{t-1})\bigg] + C_2.
    \end{align*}
    
For clarity purposes, we only show the proof with single condition $y$, but derivations can be easily extended to multiple $y$ for compositional generation, and the cases with $\hat{z}$ for manipulation.
\end{proof}

Lemma.~\ref{lemma:2} states that training unconditional diffusion models and their latent classifiers is equivalent to maximizing the ELBO of joint log-likelihood of $z$ and $y$ up to a constant.

\subsection{Connection to Linear Arithmetic}
\label{Connection to Linear Arithmetic}

The regularized guidance manipulates a given latent $\hat{z}$ in a non-linear fashion, but it degrades to linear manipulation with additional assumptions.
We take the case where there are only conjunction relations as an example and consider Eq.~(\ref{eq:reg}).

\begin{lemma}[Compositional manipulation and linear arithmetic]
When $p(z_t)$ is non-informative and $\log p(y|z_t)$ are linear, the proposed manipulation is endowed with an analytic solution 
\begin{align}
    z_0=\hat{z} + \frac{1}{\gamma_0} \sum_{i=1}^n \alpha_0^i w^i.
\end{align}
\end{lemma}
\begin{proof}
We first assume that $p(z_t)$ is a non-informative distribution where $\nabla_{z_t}p(z_t)=0$.
Then we model each $\log p(y^i|z_t)$ with a linear classifier $z \mapsto w^Tz+b$, so that the gradient $\nabla_{z_t} \log p(y^i|z_t) \sim w$ up to a scale factor\footnote{Let the scalar absorbed by $\alpha_t^i$.}.
Now when the reverse process of latent diffusion model converges at $t=0$, the whole Eq.~(\ref{eq:reg}) should converge to 0 as follows:
\begin{align}
    \sum_{i=1}^n \alpha_0^i w^i + \gamma_0 (z_0-\hat{z}) = 0,
\end{align}
which leads to the above analytic solution.
\end{proof}

For attribute negation, the solution perturbs $\hat{z}$ towards the negative direction of the classifiers.
This is a natural multi-attributes generalization of the vector arithmetic method, and we refer it as the linear version of latent classifier guidance in later comparisons.


\section{Experiments}
\label{sec:experiments}

\begin{figure*}
\centering
\begin{minipage}[c]{0.48\textwidth}
    \centering
    \includegraphics[width=.95\textwidth]{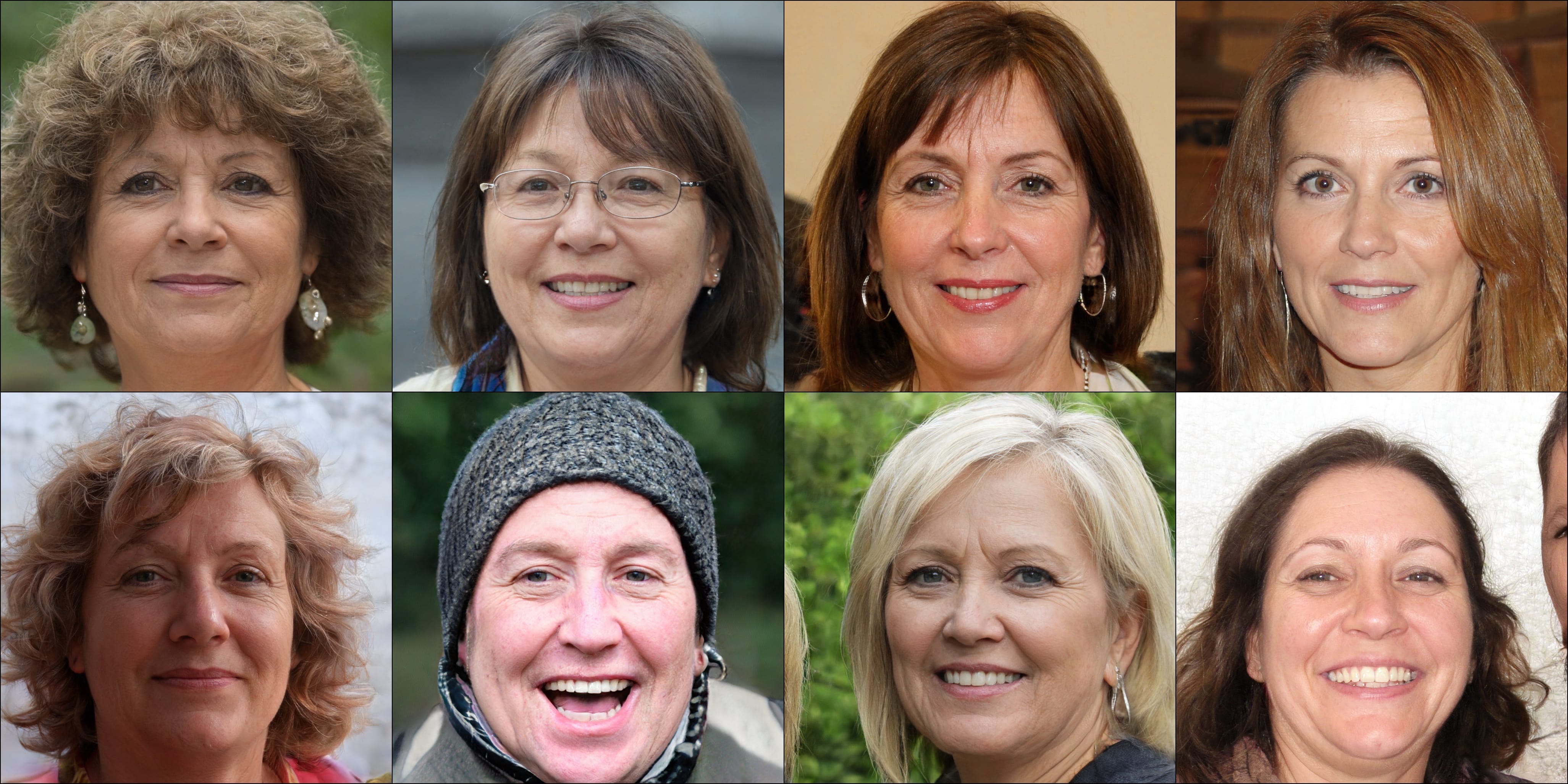}  
    \subcaption{LACE-LD~\cite{nie2021controllable}}
\end{minipage}
\begin{minipage}[c]{0.48\textwidth}
    \centering
    \includegraphics[width=.95\textwidth]{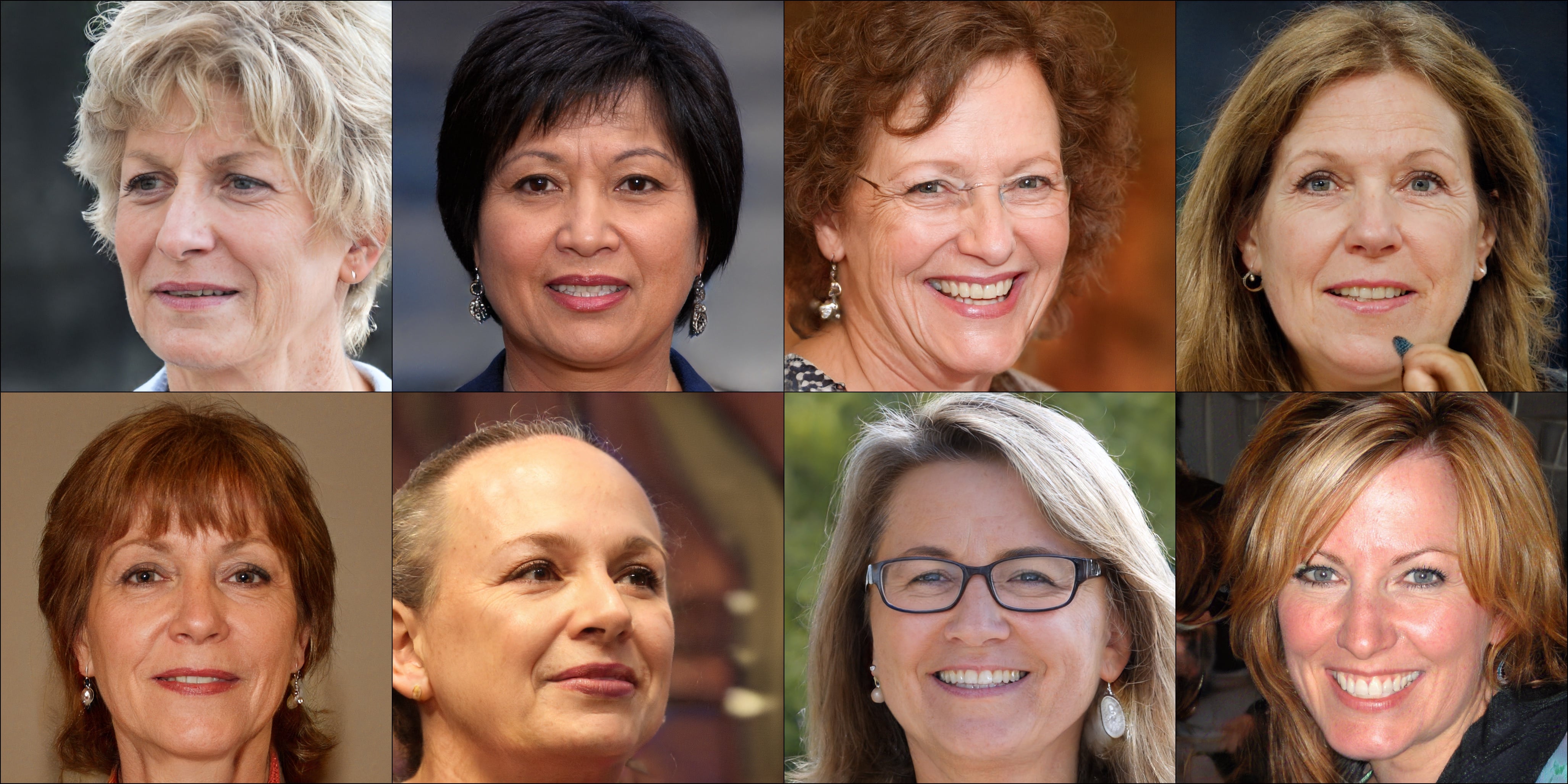} 
    \subcaption{LACE-ODE~\cite{nie2021controllable}}
\end{minipage}
\begin{minipage}[c]{0.48\textwidth}
    \centering
    \includegraphics[width=.95\textwidth]{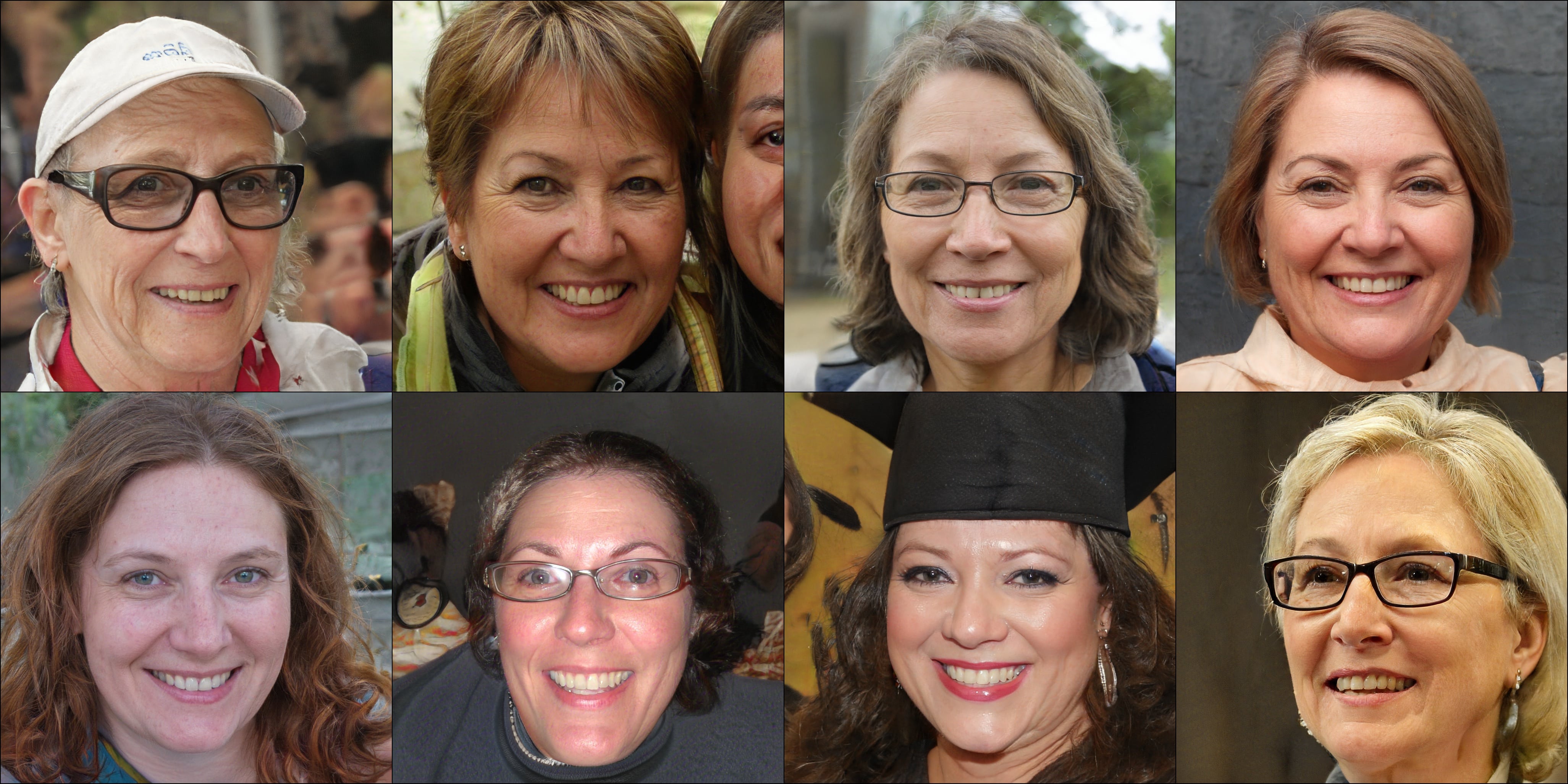}  
    \subcaption{LCG-Linear (Ours)}
\end{minipage}
\begin{minipage}[c]{0.48\textwidth}
    \centering
    \includegraphics[width=.95\textwidth]{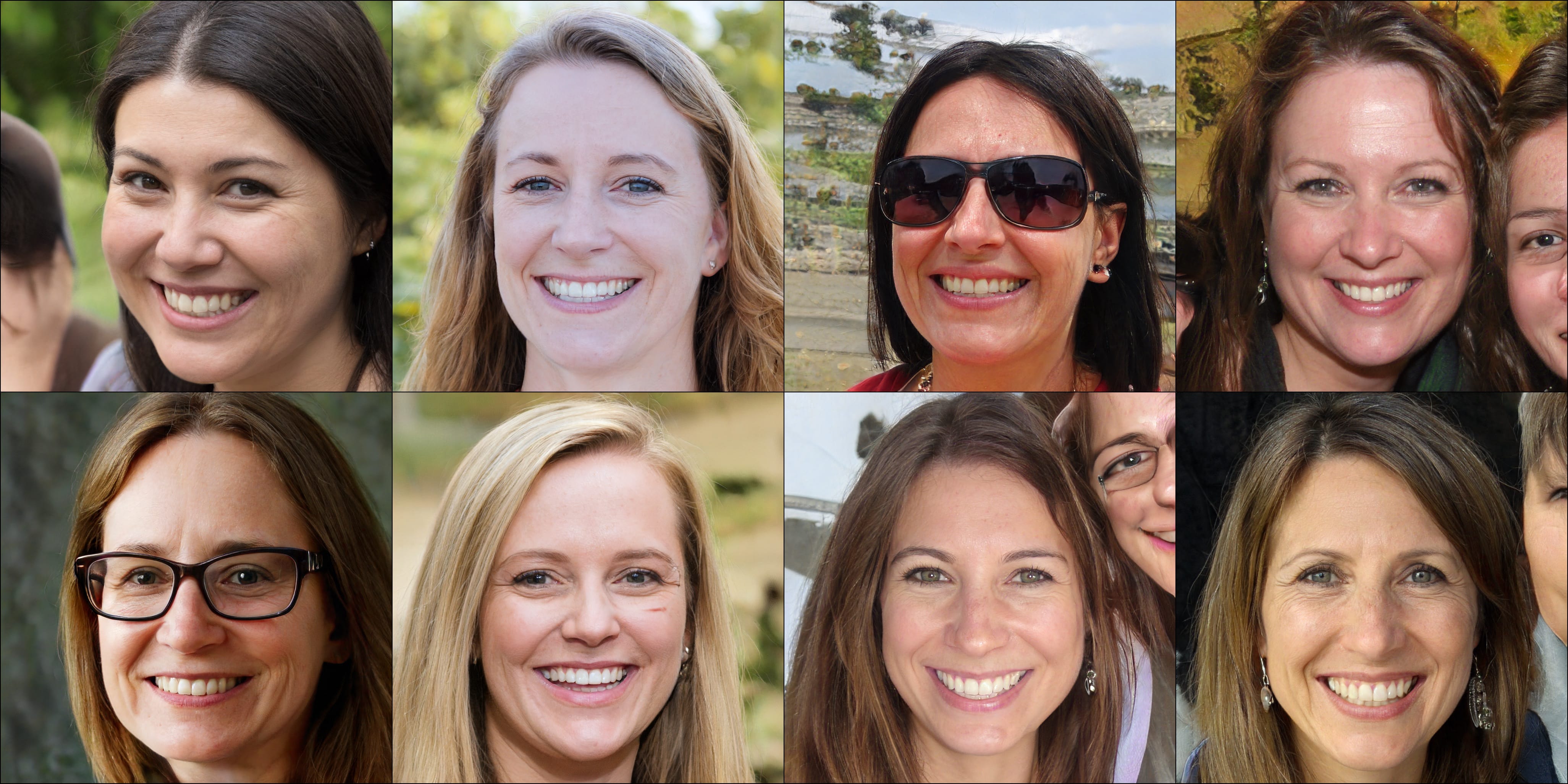}  
    \subcaption{LCG-Diffusion (Ours)}
\end{minipage}
\caption{Qualitative comparison among different methods for compositional generation based on the latenet space of the pre-trained StyleGAN2 generator under the resolution of 1024$\times$1024. The target attribute conditions are: female, smile, and 55 y/o (i.e., ``years old'').}
\label{fig:stylegan2-comp-generate}
\end{figure*}

We evaluate classifier-guided latent diffusion models for compositional generation and manipulation tasks on two pre-trained models, StyleGAN2 \cite{karras2020analyzing} and Diffusion Autoencoder \cite{preechakul2022diffusion}.
To use the described framework in the intermediate latent space ($\mathcal{W}_\text{s}$ space) of a pre-trained StyleGAN2, we first train latent DDIM \cite{song2020denoising} on 100,000 $w_\text{s}$ vectors sampled from the push-forward distribution given by the style generation.
We then train linear classifiers on the $\mathcal{W}_\text{s}$ space using the latent-label pairs provided by \cite{abdal2021styleflow}.
Note that although Eq.~(\ref{eq:conditional_elbo}) requires classifiers to be time-dependent, we find that using the same linear classifiers trained on clean $w_\text{s}$ vectors can still produce reasonable results in our preliminary experiments.
The latent diffusion model is the same as the latent DDIM used in \cite{preechakul2022diffusion}.
The performance of the classifiers can be found in Table \ref{tab:stylegan2-latent-acc}.
For Diffusion Autoencoder, we use their pre-trained latent diffusion model and linear classifiers.

For real image manipulation, we also need to encode input images to the latent space.
With StyleGAN2, we use the optimization-based inversion method in \cite{karras2020analyzing} to get the initial latent space $\mathcal{Z}_\text{s}$ and intermediate $\mathcal{W}_\text{s}$ space encodings and then employ the pre-trained pSp encoder \cite{richardson2021encoding} to get $\mathcal{W}_\text{s}+$ space encodings, where $\mathcal{W}_\text{s}+$ is a concatenation of 18 different 512-dimensional $w_\text{s}$ vectors in StyleGAN2.
With Diffusion Autoencoder, we can directly use their pre-trained encoders to get semantic vectors.

Following \cite{nie2021controllable}, we consider three metrics for our evaluation: Fréchet Inception Distance (FID) \cite{heusel2017gans}, face identity loss (ID) \cite{abdal2021styleflow} and conditional accuracy (ACC).
FID measures generation quality by comparing the Inception feature distribution of generated outputs and real images.
ID reflects the ability of a manipulation method to preserve the identity of an input face. 
A pair of input and manipulated face images are embedded by a pre-trained face recognition model\footnote{\url{https://github.com/ageitgey/face_recognition}}, and the ID score is computed as the distance between their embeddings.
ACC measures the efficacy of manipulation, which is the accuracy of classifying attributes of generated images with randomly sampled target conditions using off-the-shelf image classifiers.

\subsection{Compositional Generation}

\begin{table}[t]
    \centering
    \caption{Validation and test accuracy of linear latent classifiers of StyleGAN2.}
    \resizebox{0.95\linewidth}{!}{
    \begin{tabular}{ccc}
        \hline
        Attribute & Validation Accuracy (\%) & Test Accuracy (\%) \\
        \hline
        Smile       & 92.00 & 91.67\\
        Gender      & 93.40 & 94.20\\
        Glasses     & 92.60 & 91.30\\
        Beard       & 93.40 & 91.60\\
        Hair color  & 75.40 & 75.50\\
        Yaw         & 98.07 & 98.13\\
        Age         & 93.37 & 93.64\\
        \hline
    \end{tabular}
    }
    \label{tab:stylegan2-latent-acc}
\end{table}

\begin{figure*}
\centering
\begin{minipage}[c]{0.48\textwidth}
    \centering
    \includegraphics[width=.95\textwidth]{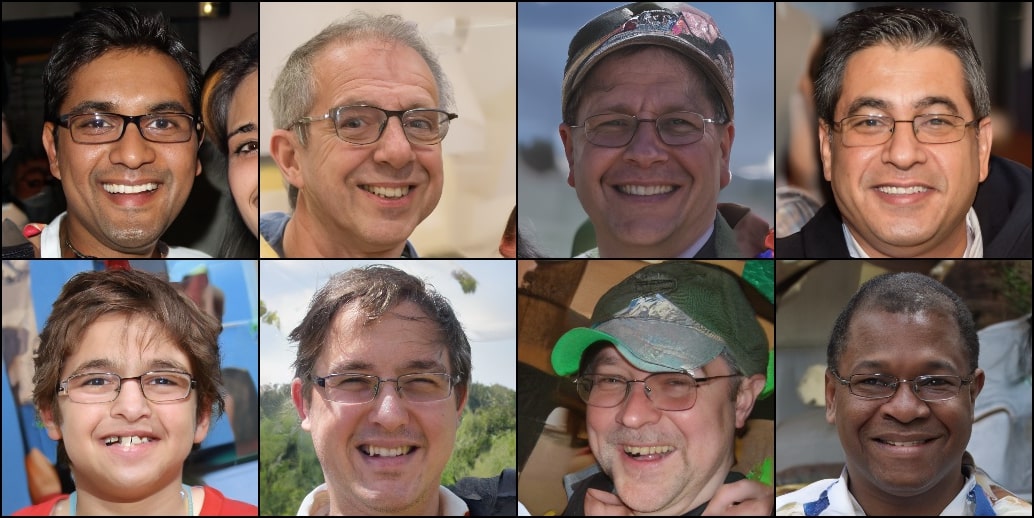}  
    \subcaption{LCG-Linear (Ours)}
\end{minipage}
\begin{minipage}[c]{0.48\textwidth}
    \centering
    \includegraphics[width=.95\textwidth]{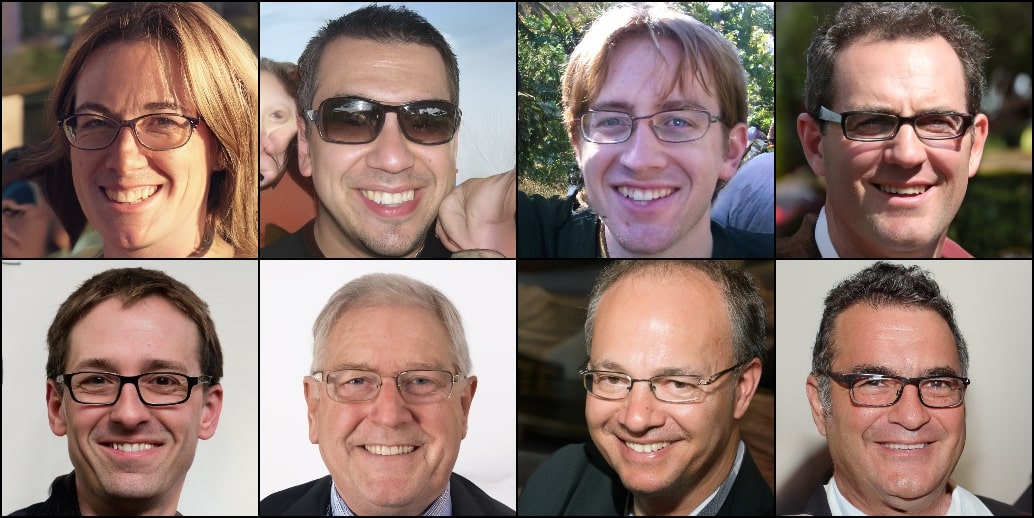}  
    \subcaption{LCG-Diffusion (Ours)}
\end{minipage}
\caption{Qualitative comparison of our methods for compositional generation based on the latent space of Diffusion Autoencoder generator under the resolution of 256$\times$256. The target attribute conditions are: male, smile, and glasses.}
\label{fig:diffae-comp-generate}
\end{figure*}

\begin{table*}[t]
\caption{Quantative comparison of different methods for compositional generation based on the latent space of pre-trained StyleGAN2.}
\vspace{-4mm}
\label{stylegan2-comparison-compositional-generate}
\begin{center}
\resizebox{0.75\textwidth}{!}{
\begin{tabular}{l|c|ccc|c|ccc}
\hline
\multirow{3}*{Method}  & \multicolumn{4}{c|}{gender, smile, age} & \multicolumn{4}{c}{-gender, smile, -haircolor}\\
\cline{2-9}
& \multirow{2}*{FID $\downarrow$} & \multicolumn{3}{c|}{ACC $\uparrow$} & \multirow{2}*{FID $\downarrow$} & \multicolumn{3}{c}{ACC $\uparrow$}\\ 
\cline{3-5}
\cline{7-9}
& & gender & smile & age & & gender & smile & haircolor  \\
\hline
StyleFlow~\cite{abdal2021styleflow}        & 43.88 & 0.718 & 0.870 & 0.874 & --- & --- & --- & ---\\
LACE-LD~\cite{nie2021controllable}          & 22.34 & 0.953 & 0.954 & \textbf{0.925} & \textbf{22.86} & 0.678 & 0.958 & 0.924\\
LACE-ODE~\cite{nie2021controllable}         & \textbf{22.03} & 0.964 & 0.967 & \textbf{0.925} & 23.51 & 0.649 & 0.970 & 0.935\\ \hline
LCG-Linear (Ours)      & 22.46 & 0.980 & \textbf{0.982} & 0.863 & 23.94 & 0.948 & \textbf{0.995} & \textbf{0.936}\\
LCG-Diffusion (Ours)  & 26.49 & \textbf{0.981} & 0.968 & 0.863 & 29.62 & \textbf{0.987} & 0.954 & 0.906\\
\hline
\end{tabular}
}
\end{center}
\label{tab:stylegan2-comp-generate}
\end{table*}

We first evaluate the ability of latent classifier guidance to generate images with multiple desired attributes.
For high-resolution images (1024$\times$1024), we select StyleGAN2 as the pre-trained generator; 
for low-resolution images (256$\times$256), we use Diffusion Autoencoder.

We compare our proposed method, which we refer as LCG (Latent Classifier Guidance) from below, with StyleFlow \cite{abdal2021styleflow} and LACE \cite{nie2021controllable}.
Results of StyleFlow are directly taken from \cite{nie2021controllable} for the conjunction of ``gender'', ``smile'' and ``age'', while it cannot handle the other compositional task where negation relations are involved.
To compare with LACE, we use their official implementation.
Note that LACE is not applicable for Diffusion Autoencoder as its semantic latent space is not endowed with a parameterized distribution as the $\mathcal{Z}_\text{s}$ or $\mathcal{W}_\text{s}$ space of StyleGAN2.
However, we can still apply latent classifier guidance because it can be easily fitted by diffusion models.

The quantitative comparison is shown in Table~\ref{tab:stylegan2-comp-generate} and the qualitative comparison shown in Fig.~\ref{fig:stylegan2-comp-generate} and Fig.~\ref{fig:diffae-comp-generate}.
While for the quantitative results target conditions are randomly sampled for each attribute, for the qualitative results, we use fixed targets for the sake of visualization.
As we can see, with latent classifier guidance, using simple linear arithmetic (LCG-Linear) and latent diffusion models (LCG-Diffusion) both perform competitively against previous non-linear methods.

\subsection{Compositional Manipulation}

We evaluate latent classifier guidance on manipulating both synthetic and real images.
\vspace{-10pt}
\paragraph{Synthetic Images}
To evaluate synthetic image manipulation, we first sample latent codes in $\mathcal{Z}_\text{s}$ space and $\mathcal{W}_\text{s}$ space, then generate their corresponding output images.
To ensure fair comparisons, we use the style network of StyleGAN2 to generate $w_\text{s}$ vectors following \cite{nie2021controllable} rather than sample vectors from the latent diffusion model that we learn.
We then sequentially edit each synthetic image given the target conditions.
Results are shown in Table~\ref{tab:comparison-seq-edit}.
Note that both of the linear arithmetic based and latent diffusion model based method achieves competitive FID and ID scores with most attributes successfully manipulated (except for ``glasses'').

\begin{table*}[t]
\caption{Quantative comparison of different methods for sequential editing with StyleGAN2.}
\vspace{-4mm}
\label{tab:comparison-seq-edit}
\begin{center}
\resizebox{0.6\linewidth}{!}{
\begin{tabular}{l|c|c|cccc}
\hline
\multirow{2}*{Method}  & \multirow{2}*{FID $\downarrow$} & \multirow{2}*{ID $\downarrow$} & \multicolumn{4}{c}{ACC $\uparrow$}\\ 
\cline{4-7}
& & &yaw & smile & age & glasses  \\
\hline
StyleFlow~\cite{abdal2021styleflow}        & 44.13 & 0.549 & \textbf{0.947} & 0.773 & 0.817 & 0.876 \\
LACE-ODE~\cite{nie2021controllable}         & 27.49 & 0.501 & 0.938 & 0.956 & \textbf{0.881} & \textbf{0.997}\\ \hline
LCG-Linear (Ours)      & 29.48 & \textbf{0.290} & 0.887 & \textbf{0.983} & {0.875} & 0.786 \\
LCG-Diffusion (Ours)  & \textbf{24.06} & 0.445 & 0.903 & 0.963 & 0.845 & 0.843 \\
\hline
\end{tabular}
}
\end{center}
\end{table*}

\vspace{-10pt}
\paragraph{Real Images}
Manipulating real images can be much harder than manipulating synthetic images as it sometimes involves inverting source images to their latent codes.
Latent classifier guidance being space-agnostic brings additional advantages when editing real images.
It is well-known that not all real images can be encoded into the $\mathcal{Z}_\text{s}$ space and $\mathcal{W}_\text{s}$ space of StyleGAN, and expanded spaces such as $\mathcal{W}_\text{s}$+ space \cite{abdal2019image2stylegan} and $\mathcal{S}$ space \cite{wu2021stylespace} are better choices for real image editing.
However, LACE is restricted to the intermediate space of StyleGAN and thus cannot leverage the richness of the expanded spaces.
Moreover, it also requires inverting input images to $\mathcal{Z}_\text{s}$ space, which is very challenging.
Latent diffusion models, on the other hand, can be trained on either existing or new expanded spaces, where the semantics are richer and the inversion is easier.

As shown in Figure~\ref{fig:stylegan2-real-edit}, latent classifier guidance outperforms LACE in terms of real image editing.
For LACE, the identities of the manipulation change dramatically in all three cases.
This is because it is generally hard to invert real input images to $\mathcal{Z}_\text{s}$ space, which is required for LACE's manipulation.
On the other hand, the latent classifier guidance only requires inversion into $\mathcal{W}_\text{s}$ or $\mathcal{W}_\text{s}+$ space and controls attributes better as well as preserves the identity more faithfully than LACE.
$\mathcal{W}_\text{s}$ space manipulation controls the attributes very well, but the image quality is sub-optimal due to the limited expressiveness of $\mathcal{W}_\text{s}$ space.
$\mathcal{W}_\text{s}+$ space manipulation provides better image quality, but the attributes are harder to control, \eg, the ``glasses'' attribute in the second row.
This is because $\mathcal{W}_\text{s}+$ space has higher dimensions and training well-behaved classifiers can be harder due to problems such as over-fitting.

\begin{figure}
\centering
\begin{minipage}[c]{0.02\textwidth}
    \centering
    \small(1)
\end{minipage}
\begin{minipage}[c]{0.10\textwidth}
    \centering
    \includegraphics[width=.95\textwidth]{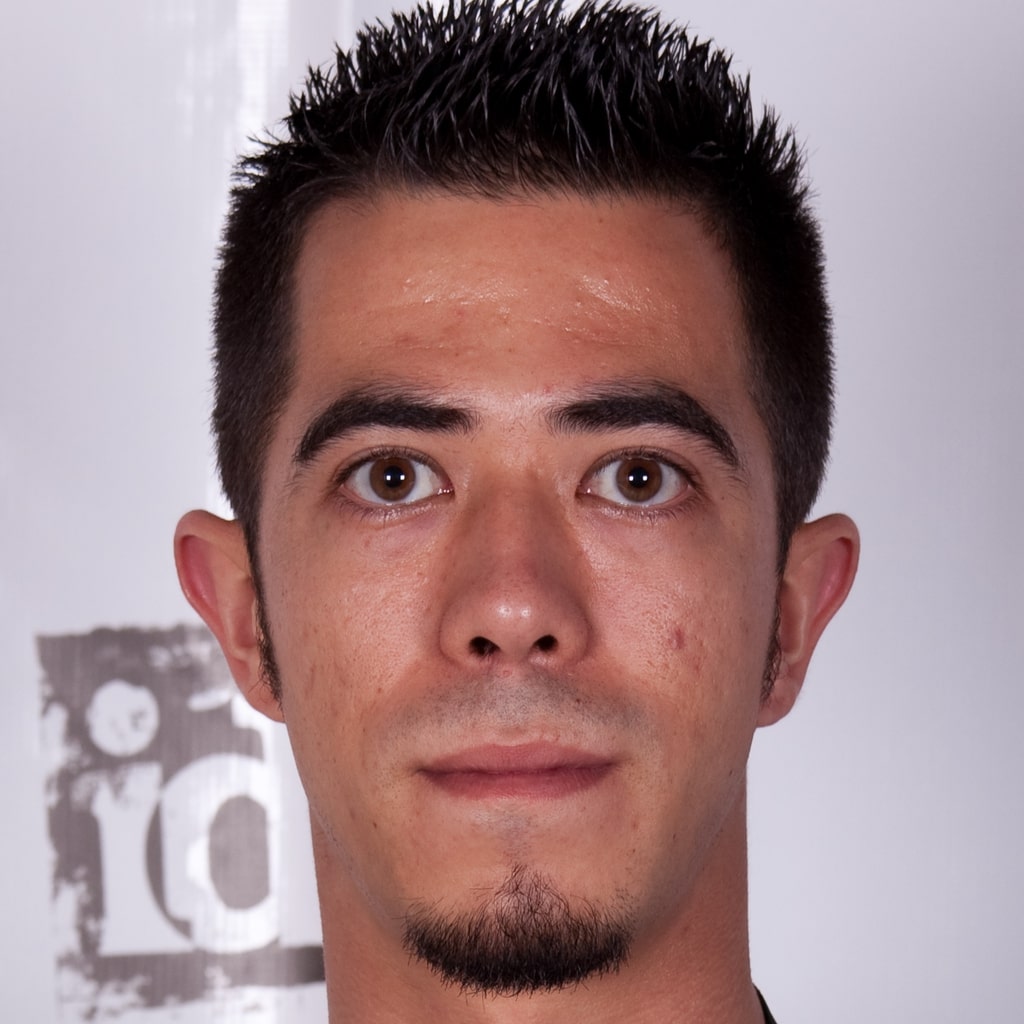}
\end{minipage}
\begin{minipage}[c]{0.10\textwidth}
    \centering
    \includegraphics[width=.95\textwidth]{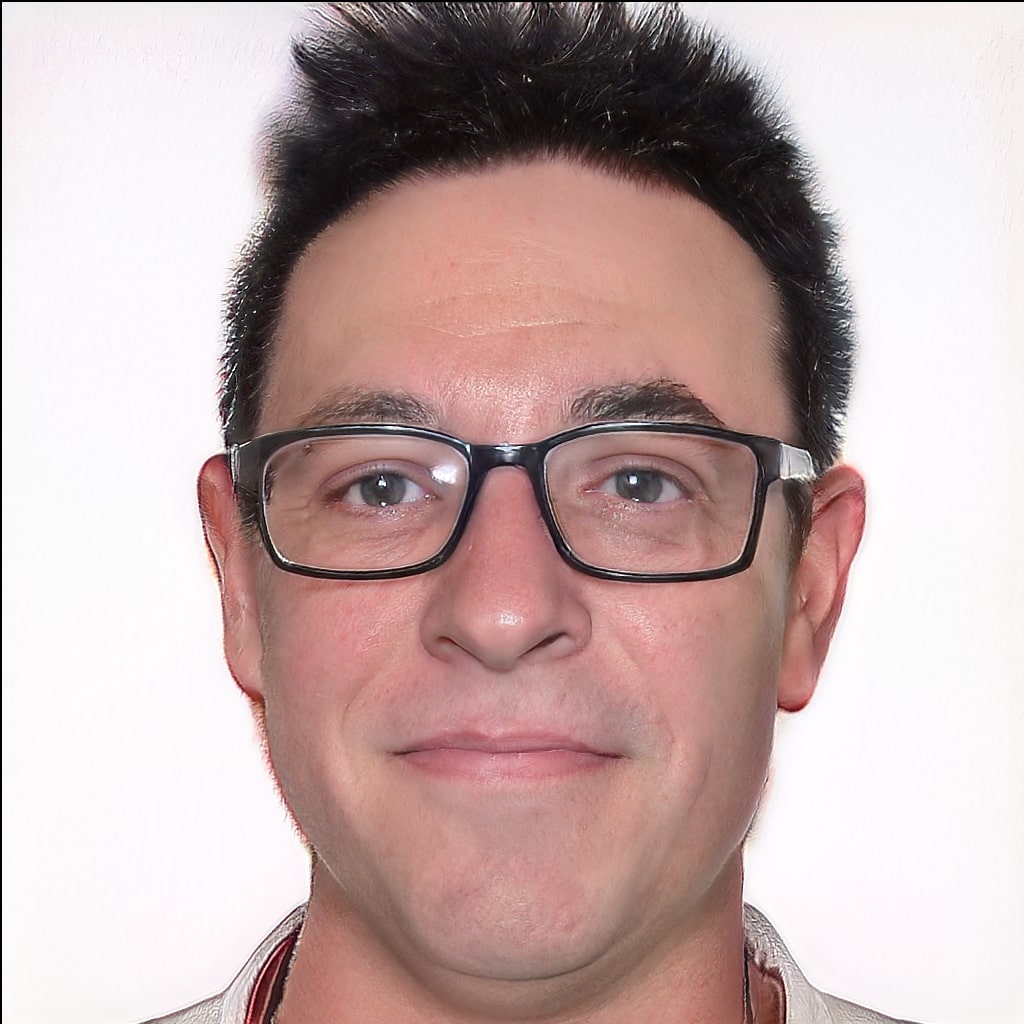}
\end{minipage}
\begin{minipage}[c]{0.10\textwidth}
    \centering
    \includegraphics[width=.95\textwidth]{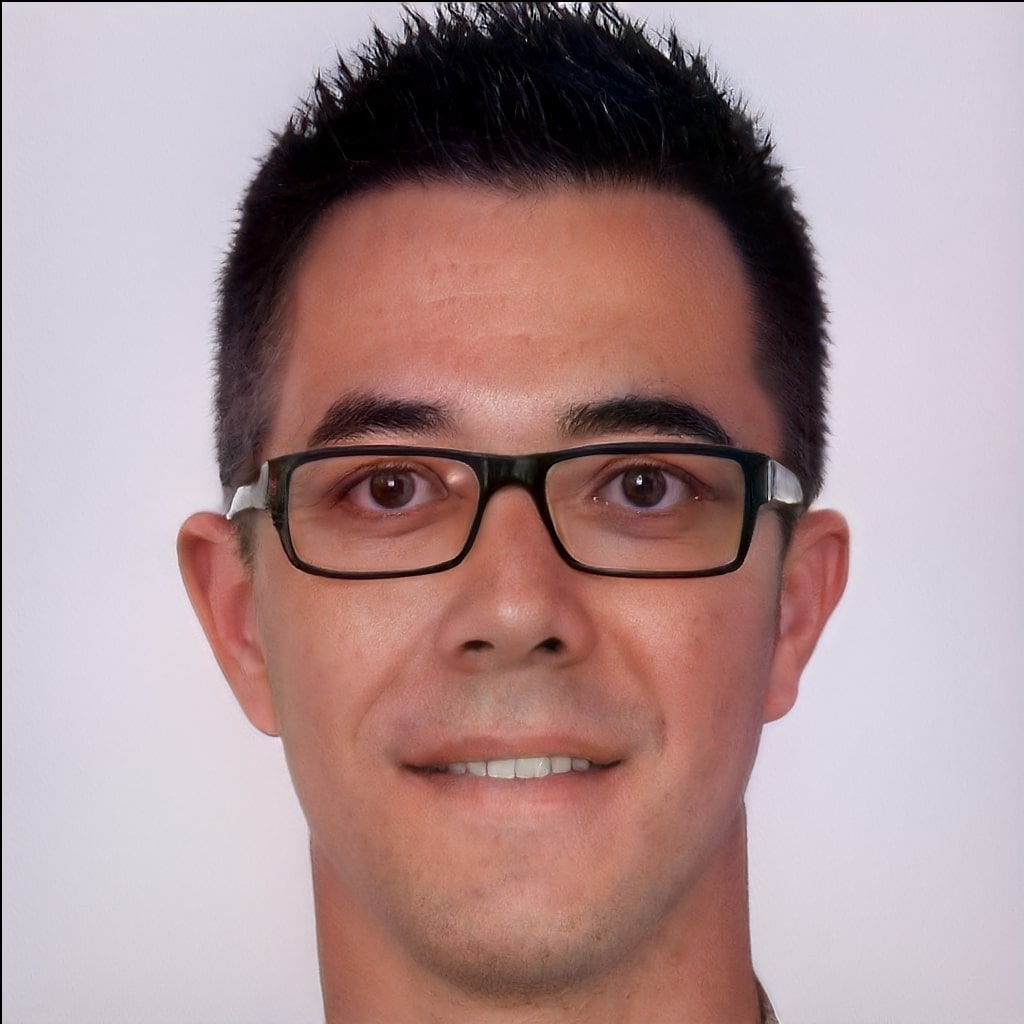}
\end{minipage}
\begin{minipage}[c]{0.10\textwidth}
    \centering
    \includegraphics[width=.95\textwidth]{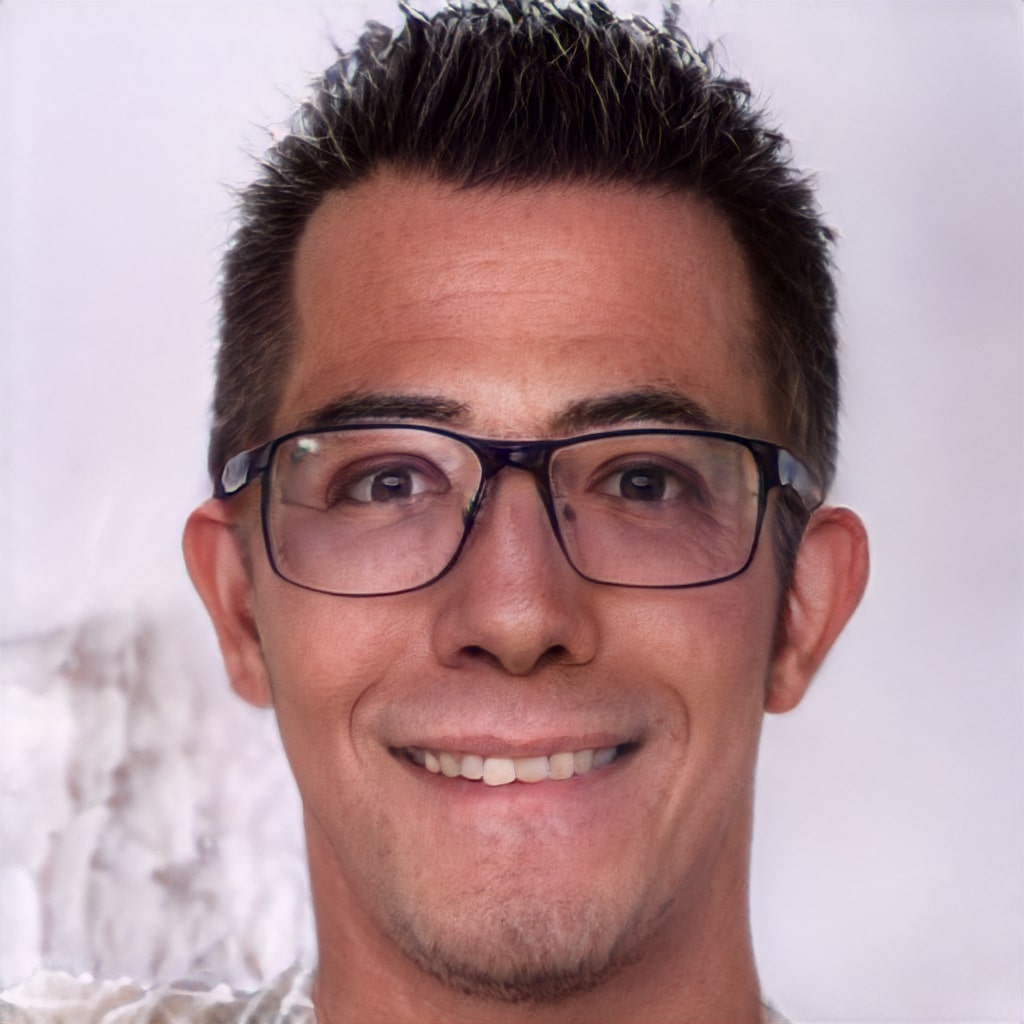}
\end{minipage}\\

\begin{minipage}[c]{0.02\textwidth}
    \centering
    \small(2)
\end{minipage}
\begin{minipage}[c]{0.10\textwidth}
    \centering
    \includegraphics[width=.95\textwidth]{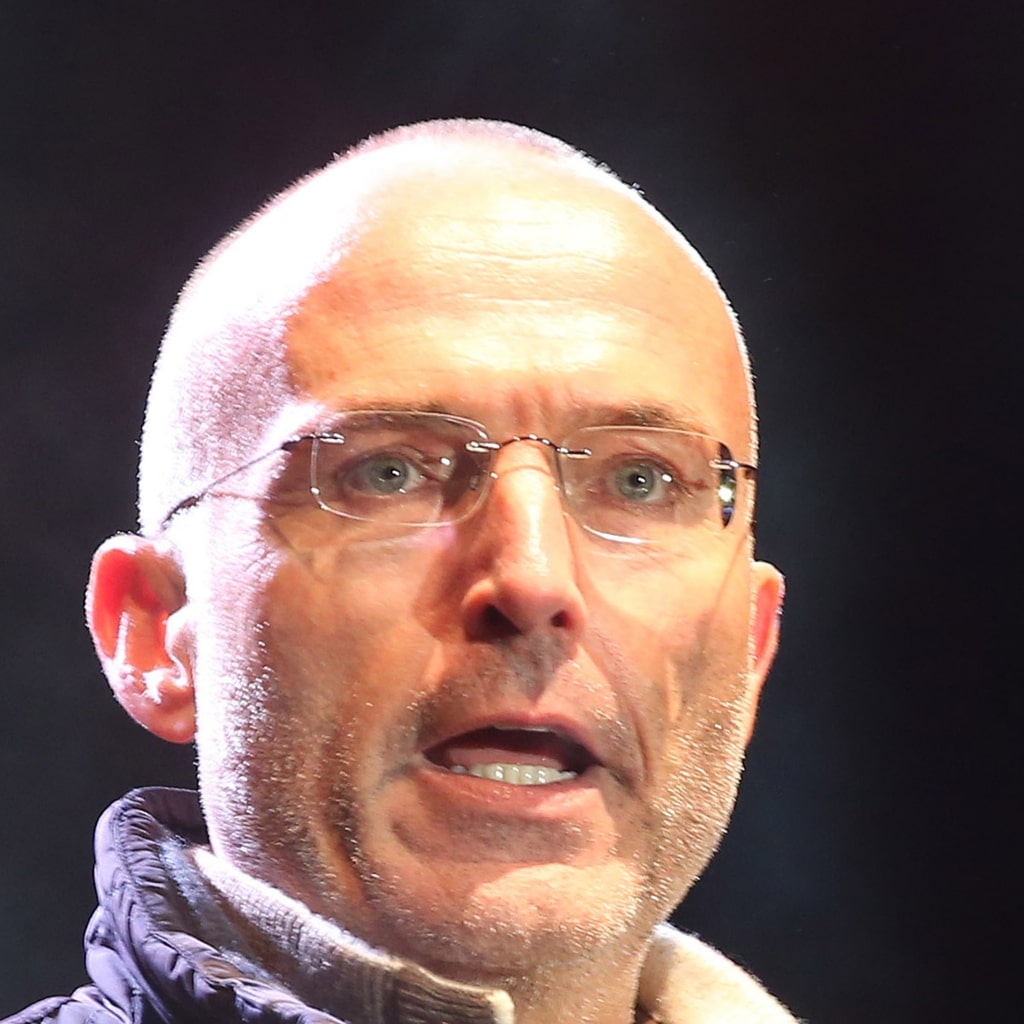}
\end{minipage}
\begin{minipage}[c]{0.10\textwidth}
    \centering
    \includegraphics[width=.95\textwidth]{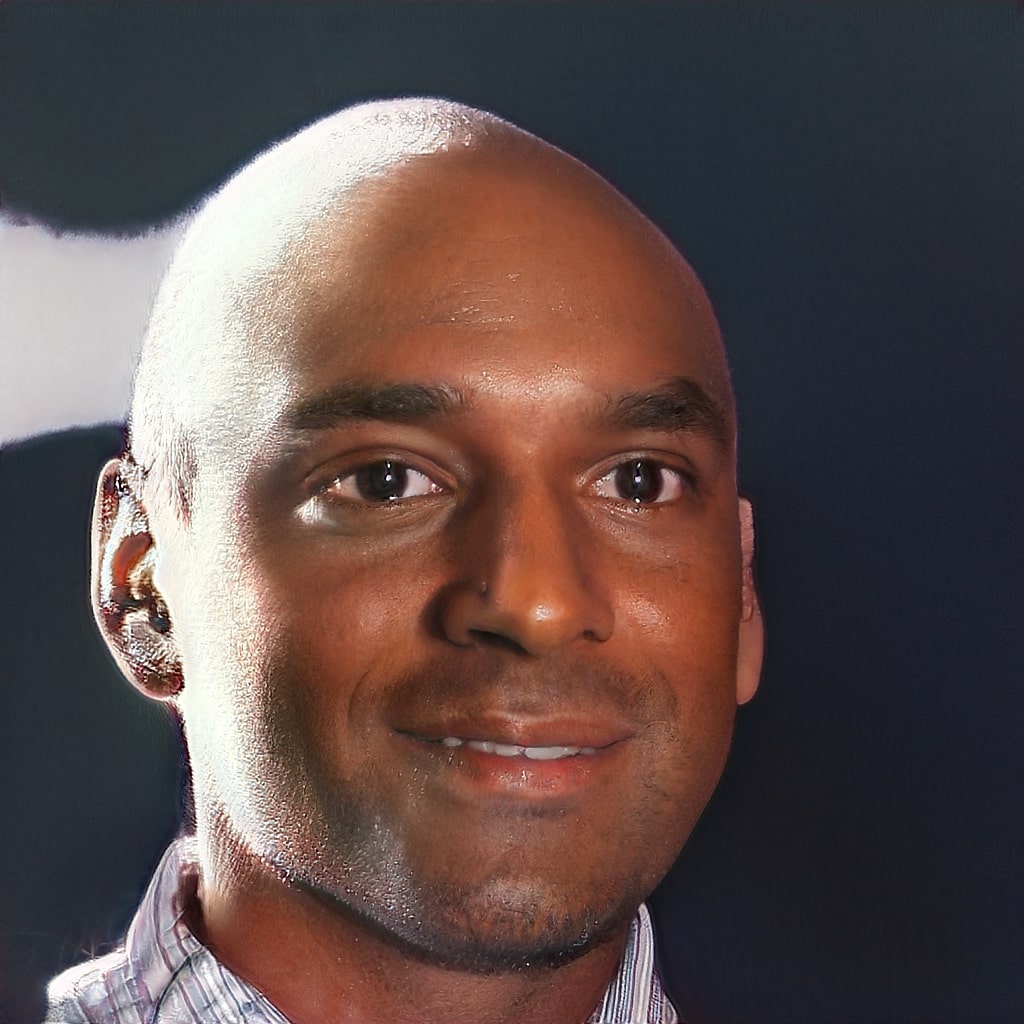}
\end{minipage}
\begin{minipage}[c]{0.10\textwidth}
    \centering
    \includegraphics[width=.95\textwidth]{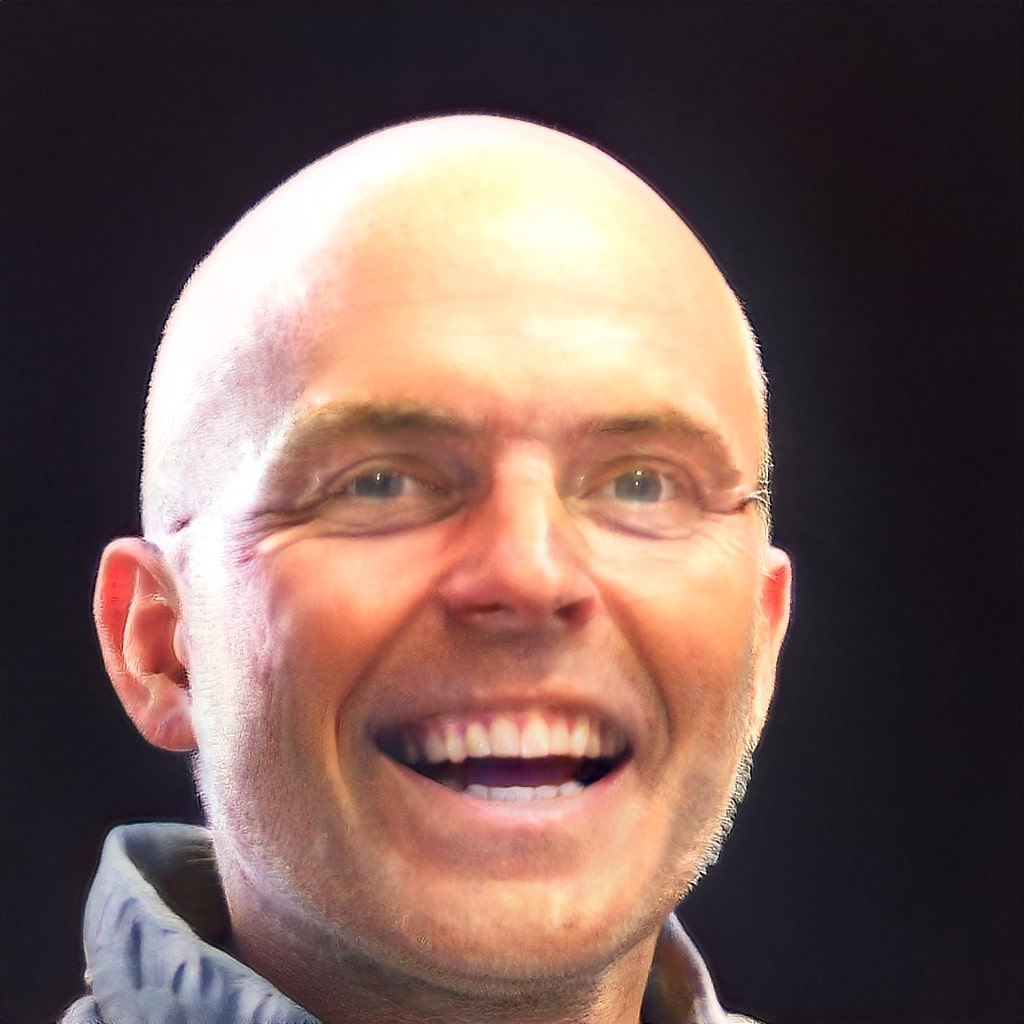}
\end{minipage}
\begin{minipage}[c]{0.10\textwidth}
    \centering
    \includegraphics[width=.95\textwidth]{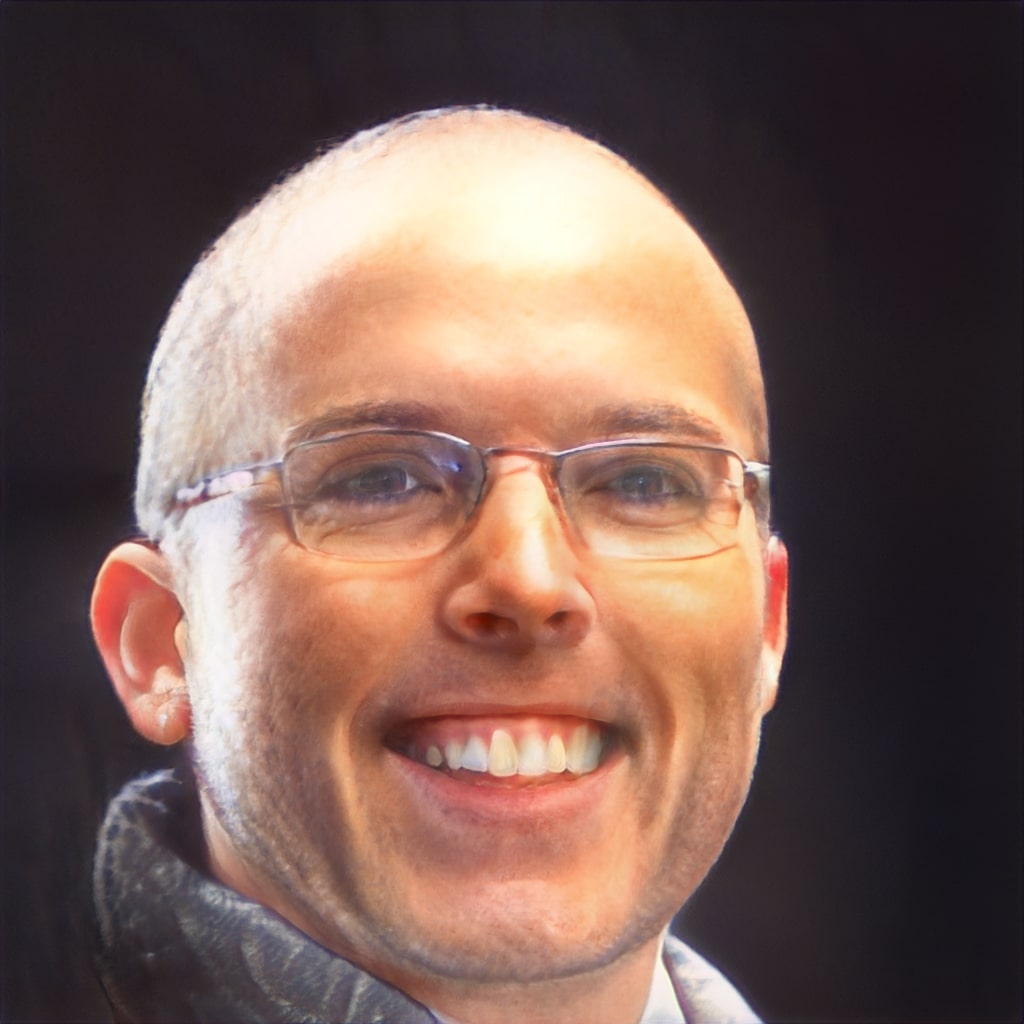}
\end{minipage}\\

\begin{minipage}[c]{0.02\textwidth}
    \centering
    \small(3)
\end{minipage}
\begin{minipage}[c]{0.10\textwidth}
    \centering
    \includegraphics[width=.95\textwidth]{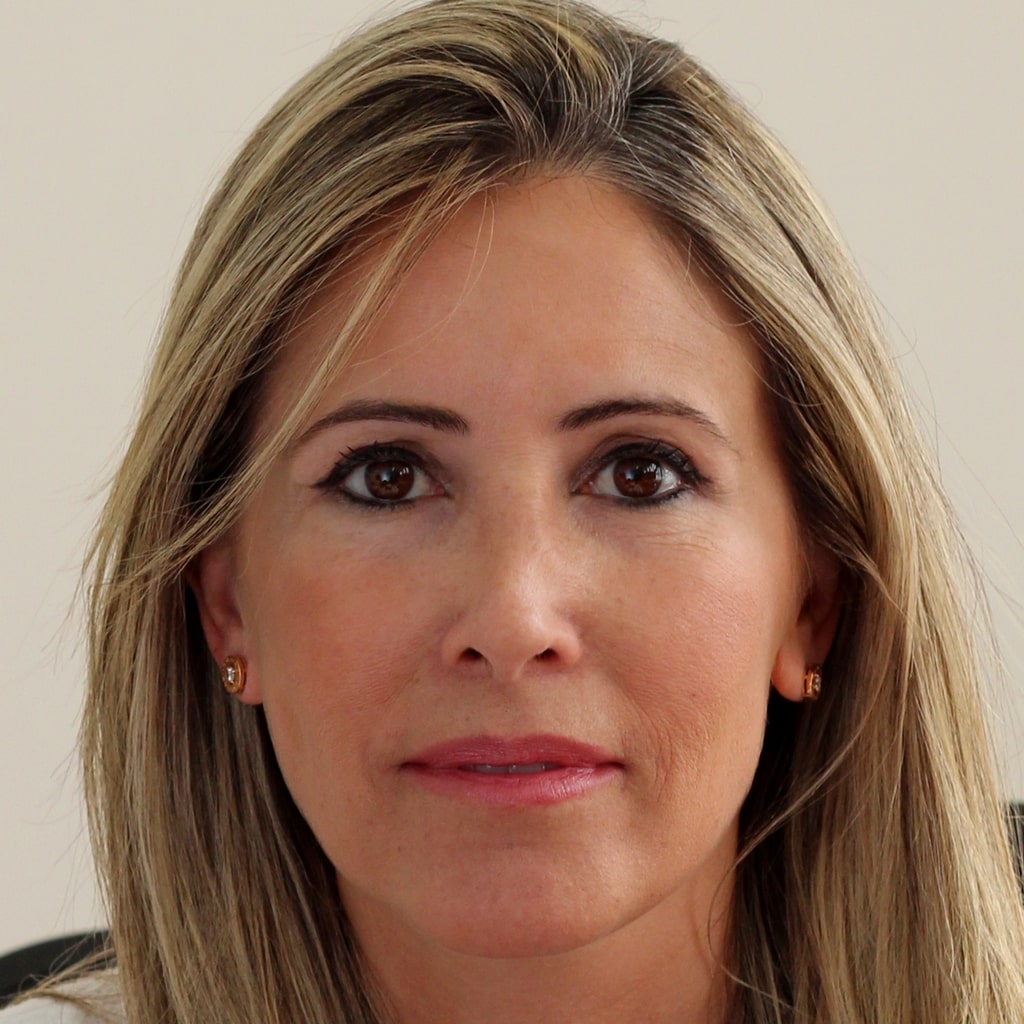}  
    \subcaption{}
\end{minipage}
\begin{minipage}[c]{0.10\textwidth}
    \centering
    \includegraphics[width=.95\textwidth]{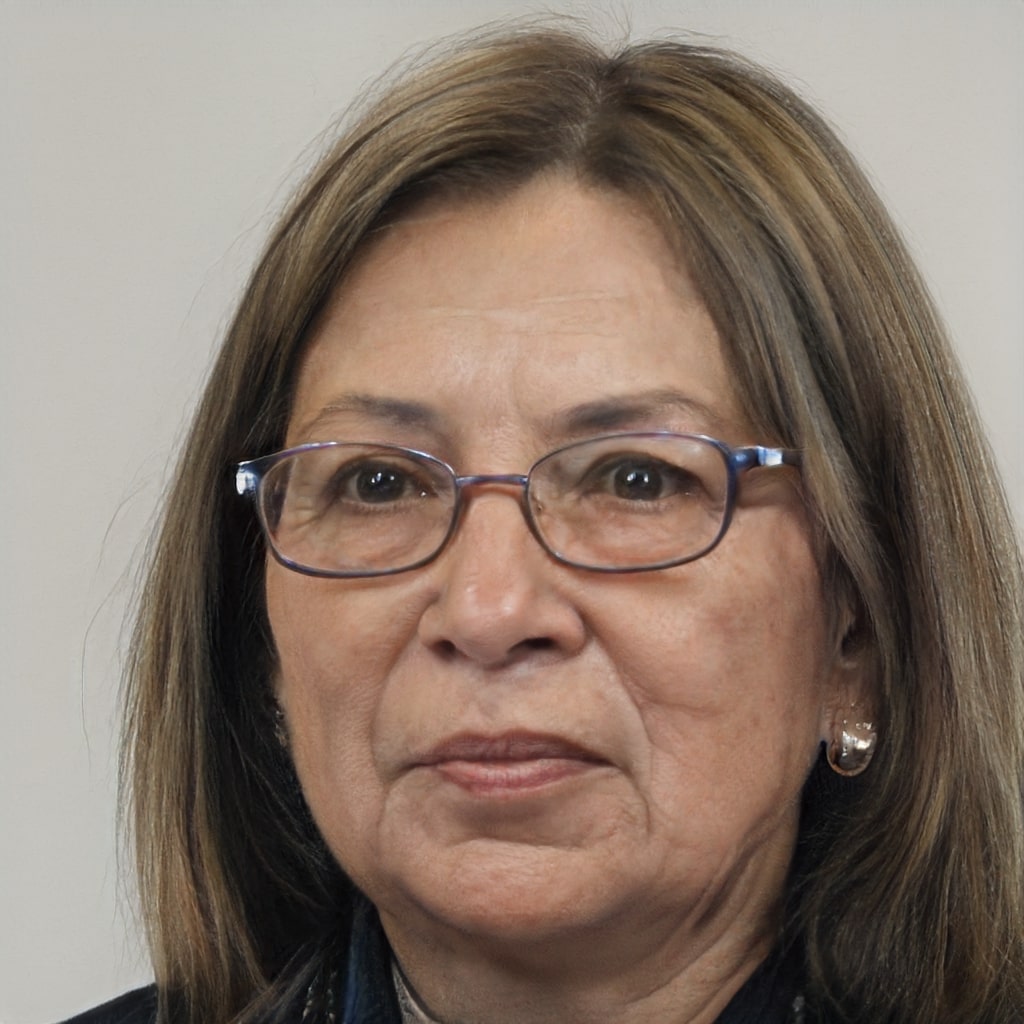}  
    \subcaption{}
\end{minipage}
\begin{minipage}[c]{0.10\textwidth}
    \centering
    \includegraphics[width=.95\textwidth]{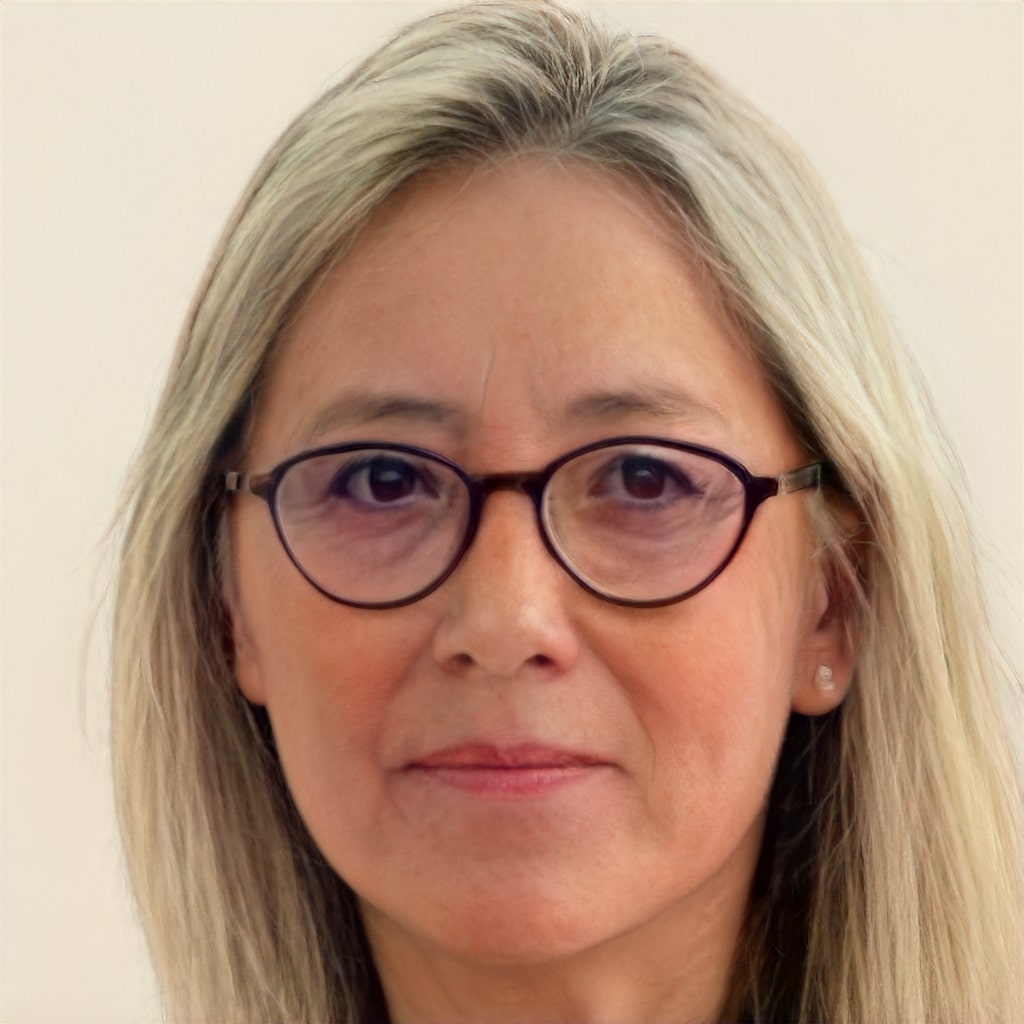}  
    \subcaption{}
\end{minipage}
\begin{minipage}[c]{0.10\textwidth}
    \centering
    \includegraphics[width=.95\textwidth]{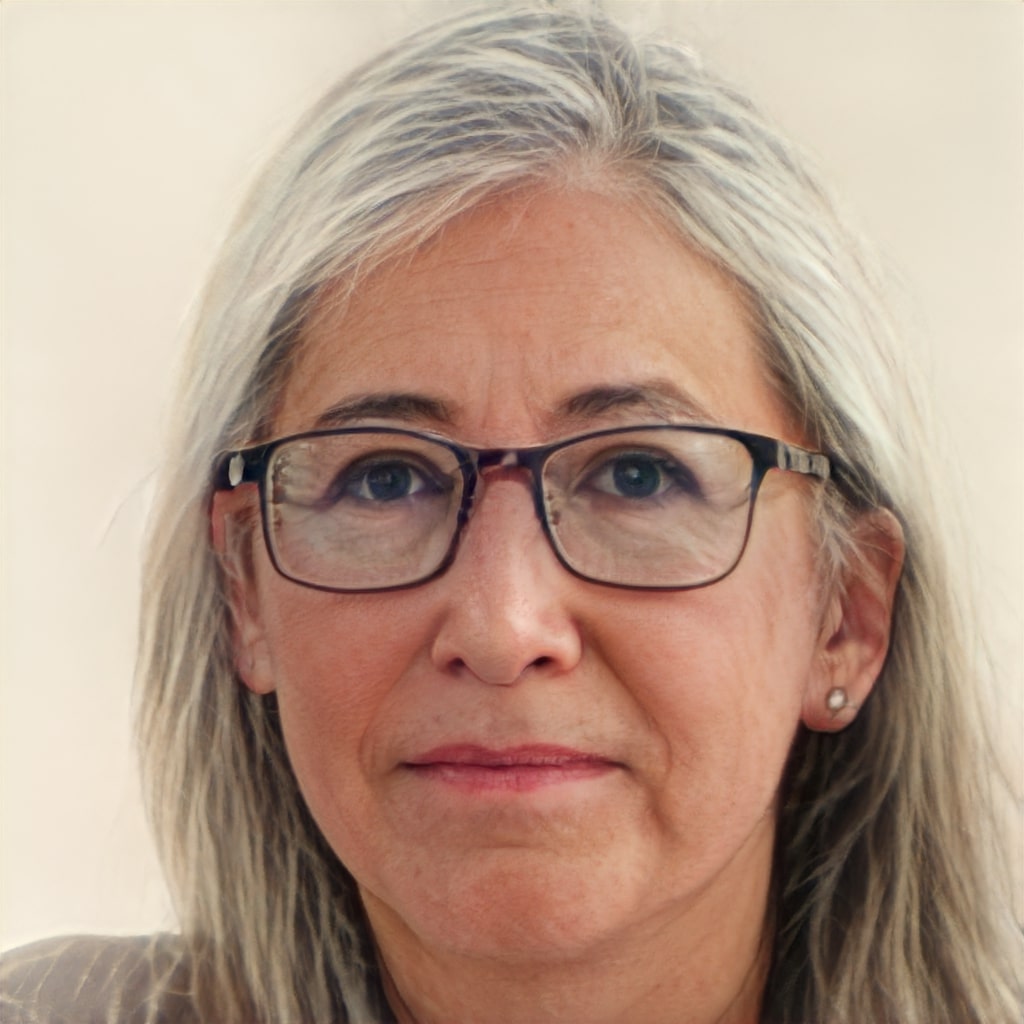}  
    \subcaption{}
\end{minipage}
\caption{Qualitative comparison among different compositional manipulation methods on real image inputs.
Target attributes from the top row to the bottom row are 
(1) glasses, smile, 55 y/o, 
(2) no glasses, smile, 28 y/o,
(3) glasses, white hair, right face.
Each column from left to right are
(a) original images,
(b) LACE-ODE \cite{nie2021controllable}, 
(c) linear latent classifier guidance in $\mathcal{W}_\text{s}$, 
(d) linear latent classifier guidance in $\mathcal{W}_\text{s}$+.}
\label{fig:stylegan2-real-edit}
\end{figure}


\section{Discussion}

\subsection{Why is the linear method competitive?}

\begin{figure}
    \centering
    \includegraphics[width=.9\linewidth]{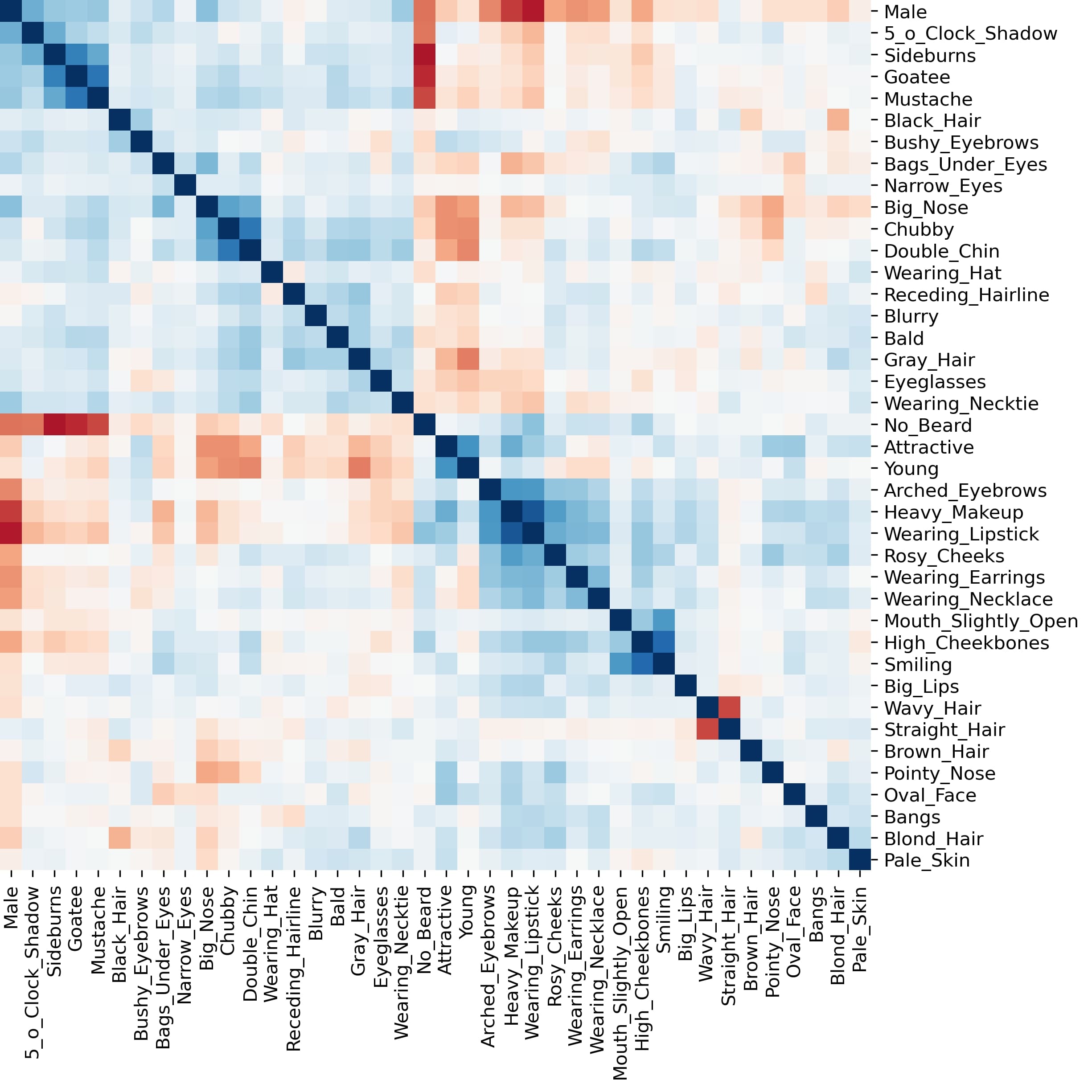}
    \caption{Visualization of Latent semantic correlations of Diffusion Autoencoder.
    Bluer regions indicate positive correlations and redder areas suggest negative correlations.}
    \label{fig:diffae-cls-corr}
\end{figure}

One main challenge of manipulating multiple attributes is maintaining the non-targetted attributes.
For compositional generation, these are the out-of-scope attributes; for sequential editing, these also include the previously manipulated attributes.
To tackle the challenge, previous methods either design a protection with linear manipulation such as InterfaceGAN \cite{shen2020interpreting}, or use non-linear manipulation such as StyleFlow and LACE.
As we have shown in previous sections, our LCG-linear can be very competitive against other non-linear methods despite its simplicity.

To understand its power, we examine the linear classifiers that are learned for manipulation.
Here we use the heatmap to visualize the correlation between each pair of linear classifiers learned on the semantic latent space of Diffusion Autoencoder, as shown in Fig.~\ref{fig:diffae-cls-corr}.
As we can see, the semantic latent space of Diffusion Autoencoder is favorably disentangled and thus the linear classifiers show strong orthogonality frequently.
This means that even without a specific protection mechanism as in \cite{shen2020interpreting}, LCG-linear is still capable of preserving the identity in most cases.
In Table~\ref{tab:stylegan-attr-changes}, we list the changes of conditional accuracy of other attributes when a single attribute is linearly manipulated.
The four edits correspond to the four attributes ``yaw'', ``smile'', ``age'', and ``glasses'' respectively.
The small changes indicate that the attributes are well disentangled in such generative models, and further explain the efficacy of LCG-linear.

\begin{table}[h]
    \centering
    \caption{Changes of conditional accuracy for each linear sequential editing.}
    \label{tab:stylegan-attr-changes}
    \resizebox{0.8\linewidth}{!}{
    \begin{tabular}[t]{c|cccc}
    \hline
    & yaw & smile & age  & glasses \\
    \hline
    Edit-1  & ---    & -0.001 & +0.001 & -0.002 \\
    Edit-2  & +0.003 & ---    & +0.001 & +0.007 \\
    Edit-3  & +0.001 & +0.000 & ---    & -0.013 \\
    Edit-4  & +0.000 & -0.012 & -0.005 & ---    \\
    \hline
    \end{tabular}
    }
\end{table}

\subsection{When is the non-linear method preferred?}

Despite the complicity of the non-linear diffusion based method, it does not always perform favorably against the linear version.
The main motivation for non-linear methods, as argued in \cite{abdal2021styleflow}, is that linear manipulation often moves a latent code outside the latent distribution which leads to low-quality generation. 
Indeed, in linear manipulation, we assume a non-informative $p(z_t)$ which favors different regions of sample space equally, regardless of the actual latent distribution.
This indicates that non-linear control is likely to profit when the generation needs to traverse low density region or is simply out-of-distribution.

An example is sequential editing. 
Sequential editing is more prone to low density region, as the new edit is conditioned on previous edits that possibly have already guided the latent to low density regions.
To see this, imagine a 3-D Guassian distribution where each axis represents an attribute and we want to guide a sample point from $[-1,-1,-1]$ to $[1,1,1]$.
Compositional generation is analogous to a direct path $[-1,-1,-1] \rightarrow [1,1,1]$, while sequential editing is analogous to a path $[-1,-1,-1] \rightarrow [-1,-1,1] \rightarrow \rightarrow [-1,1,1] \rightarrow [1,1,1]$ that traverses more low density regions.
As we can in Table~\ref{tab:comparison-seq-edit}, LCG-diffusion outperforms LCG-linear on FID, generating more realistic images.
Characterizing the latent distribution with diffusion model in this case is favorable than a non-informative one, as the diffusion model always pulls the sample toward high density region thus preventing it going out-of-distribution.
A downside of this is that, as we can see from the ID score, keeping images realistic is at the cost of losing identity preservation.
Improving the identity preservation should serve as an interesting topic for future exploration.

\section{Conclusion and Future Work}

In conclusion, we study the efficacy of using latent diffusion models and latent classifier guidance for compositional visual generation and manipulation.
Specifically, we train latent diffusion models and auxiliary latent classifiers to facilitate non-linear navigation of latent representation generation for two pre-trained generative models, StyleGAN2 and Diffusion Autoencoder. 
We demonstrate such a paradigm is suitable for compositional visual tasks both theoretically and empirically. 
Our findings suggest that latent classifier guidance is a promising approach that deserves further research, even in the presence of other strong methods such as Stable Diffusion \cite{rombach2022high}.

In our future work, we plan to explore modeling more complicated relations between attributes and aim to achieve compositionality on more challenging datasets and tasks, such as text/class-conditioned video generation \cite{balaji2019conditional,ho2022video,ni2023conditional,singer2022make}.
We are also interested in the performance of latent classifier guidance in out-of-distribution settings.
In addition, we acknowledge that relying on a pretrained generative model with a semantic latent space may not be practical or feasible in some scenarios.
To address this potential problem, we will explore the possibility of reorganizing the latent space of the pretrained generative model to construct a semantic latent space in the post-pretraining stage. Also, the demand for generating unseen classes and unseen sub-concept of an existing class is of surged interest in the community \cite{kumari2022multi,nitzan2023domain,kumari2023ablating,ruiz2022dreambooth}. To tackle these new challenges for compositional generation, we plan to investigate the applicability of leveraging the continual learning and incremental learning techniques \cite{van2020brain,volpi2021continual,liang2022balancing,wang2022learning,wang2022dualprompt,wang2022sprompts,tong2023incremental} to extend the semantic latent space of the generative model accordingly.


{\small
\bibliographystyle{ieee_fullname}
\bibliography{cvprw2023/latex/egbib}
}

\end{document}